\documentclass[twoside,11pt]{article}
\usepackage{color}
\usepackage{jmlr2e}
\usepackage{diagbox}

\RequirePackage[OT1]{fontenc}
\RequirePackage[colorlinks,citecolor=blue,urlcolor=blue]{hyperref}
\RequirePackage{hypernat}
\RequirePackage{soul}
\usepackage{amsfonts}
\usepackage{dcolumn}
\usepackage{textcomp}
\usepackage{longtable}
\usepackage{verbatim}
\usepackage{latexsym}

\usepackage{multirow}
\usepackage[linesnumbered,ruled,vlined]{algorithm2e}
\usepackage{url}

\usepackage{algorithmic}

\usepackage{comment}

\usepackage{float}

 \newtheorem{assumption}{Assumption}

 \newtheorem{mlemma}{Lemma}

\def\ind{\begin{picture}(9,8)
         \put(0,0){\line(1,0){9}}
         \put(3,0){\line(0,1){8}}
         \put(6,0){\line(0,1){8}}
         \end{picture}
        }
\def\nind{\begin{picture}(9,8)
         \put(0,0){\line(1,0){9}}
         \put(3,0){\line(0,1){8}}
         \put(6,0){\line(0,1){8}}
         \put(1,0){{\it /}}
         \end{picture}
    }

\begin{document}

\title{Causal Network Learning from Multiple Interventions of Unknown Manipulated Targets}

\author{\name Yang-Bo He  \email heyb@math.pku.edu.cn \\
         \name Zhi Geng  \email zgeng@math.pku.edu.cn \\
        \addr School of Mathematical Sciences, LMAM, LMEQF, \\ Center for Statistical Science, Peking University\\
       Beijing 100871, China
}

\editor{ }

\maketitle

\begin{abstract}
In this paper, we discuss structure learning of causal networks from multiple data sets obtained by external intervention experiments where we do not know what variables are manipulated.
For example, the conditions in these experiments are changed by changing temperature
or using drugs,
but we do not know what target variables are manipulated by the external interventions.
From such data sets, the structure learning becomes more difficult.
For this case, we first discuss the identifiability of causal structures.
Next we present a graph-merging method for learning causal networks
for the case that the sample sizes are large for these interventions.
Then for the case that the sample sizes of these interventions are relatively small,
we propose a data-pooling method for learning causal networks in which
we pool all data sets of these interventions together for the learning.
Further we propose a re-sampling approach to evaluate the edges of the causal network
learned by the data-pooling method.
Finally we illustrate the proposed learning methods by simulations.
\end{abstract}

 \section{Introduction}

Directed acyclic graphs (DAGs) can be used to represent causal networks among variables.
Many methods have been developed to learn the structures of DAGs
from observational and/or experimental data 
\citep{pearl1995causal,geng2004relationship,Friedman, maathuis2009estimating,jansen2003bayesian,
finegold2011robust, heckerman1999bayesian,cooper1999causal}.
Unlike an observational study,
we can externally manipulate a few of target variables in an intervention experiment.
Thereafter the variables which are manipulated in an experiment are simply called targets.
The score-based and constraint-based methods are available to learn casual networks from interventional data when the targets in intervention experiments are known.
For examples,
 \citet{cooper1999causal} present a Bayesian method of causal discovery from a mixture of
experimental and observational data.  \citet{eberhardt2006sufficient} discusses independence   test used in the constraint-based methods via the data from multiple interventions and shows that two data sets obtained from two interventions with different targets can be pooled to test the conditional independence of two variables $x_1$ and $x_2$ given a variable set $S$ if
$S$ separates $(x_1, x_2)$ from all targets in the two interventions. \citet{lagani2012learning} show that the data pooling is valid for testing independencies with those interventions in which all targets except one are manipulated to the same value across the interventions. {\citet{hauser2012characterization} discuss the graph representation of Markov equivalence class  under interventions and causal structure learning from  multiple intervention experiments.}

{In some applications,
the statuses or values of variables are stable and even keep constant
in the normal condition and environment.
Thus to discover the causal relationships among these variables,
we have to manipulate a few of variables or
change the condition or environment
such that these variables change their statuses and values
and affect their effect variables.}
In some situations, we may not know what target variables are manipulated in intervention experiments.
For examples, when experiments are implemented by changing temperature or by using some medicine,
we may not know exactly the targets of these interventions.  {To deal these situations,
\citet{eaton2007exact}   introduce a vertex for each intervention and use   DAGs  over the regular  and intervention vertices to represent the causal relationships among regular vertices and the targets of the interventions. They apply the dynamic programming algorithm introduced in \citet{koivisto2004exact}  to  computes the exact posterior marginal edge probabilities of the DAGs.
{As they mentioned,
their computation is limited to about 20 vertices due to
the space and time limits.}
To our knowledge, there are still many  unresolved issues left when the targets of intervention are unknown, such as the identification of causal structures and the learning methods for large causal networks.}

In this paper, we focus on the constraint-based causal learning methods using data from multiple  interventions with unknown manipulated targets. We first discuss the identifiability of causal structures. Then for the case that the sample size from each intervention is large, we propose to learn a network from each intervention data set and merge these learned networks.
This method can learn more directed edges from intervention data sets than
from an observational data set even if we do not know the targets of interventions.
Next when the sample sizes are small, the statistical errors of tests for each small intervention data set cannot be neglected, and thus we pool all intervention data sets together
to learn a network structure, and then we use re-sampling technique to evaluate
the edges of the learned network. We discuss the identifiability of causal structures
learned from the pooled data and show that the proposed data-pooling method
can correctly learn some local structures of the underlying causal network.

The rest of the paper is arranged as follows. In Section \ref{chapnotations}, we introduce the notation of causal network model and discuss the  causal structure learning with interventions. In Section \ref{chaplearning}, we propose two methods of causal network learning from multiple interventions with unknown targets. We
evaluate the proposed methods via simulations in Section \ref{simulation}.
Finally we discuss these methods in Section \ref{disccussion}. 

\section{Causal network model and causal learning with interventions} \label{chapnotations}
In this section, we first introduce notation and assumptions of causal network models, and  then discuss  causal structural learning with  interventions.

\subsection{Causal network model}
 A \emph{directed acyclic graph} (DAG)    ${G}=({{X}},E)$ is used to represent the causal relationships of vertices, where ${X }=\{x_1,\cdots,x_p\}$
denotes a vertex set and $E$ denotes a set of directed
edges.
{For a directed edge $x_i \rightarrow x_j$ in a DAG $G$,
$x_i$ is a \emph{parent} vertex of $x_j$ and
$x_j$ is a \emph{child} vertex of $x_i$;
we also interpret $x_i$ as a cause of $x_j$
and the vertex $x_j$ as a effect of $x_i$.}
\emph{A directed path} from $x_1$ to $x_k$ in $G$ is a sequence of directed edges that connect $x_i$ and $x_{i+1}$ {($x_i \to x_{i+1}$)} for $i = 1,\cdots,k-1$. A vertex $x_j$ is a \emph{descendant} of $x_i$ if there is at least a directed path from $x_i$ to $x_j$ in $G$; otherwise, $x_j$ is \emph{non-descendant} of $x_i$. We use $pa(x_i)$, $ch(x_i)$ and $nd(x_i)$ (or simply $pa_i$, $ch_i$ and $nd_i$) for  the sets of parents, children, and non-descendants of a vertex $x_i$, respectively. A graph $G'=(X,E')$ is an \emph{edge-deleted subgraph} of $G=(X,E)$ if $E' \subset E$. The  \emph{skeleton} of $G$ is an undirected graph
{obtained by replacing all directed edges
in $G$ with the corresponding undirected edges.}
{A three-vertex structure $x_i\to x_j \leftarrow x_k$ is
called} a \emph{v-structure} if neither $x_i\to x_k$ nor $x_i\leftarrow x_k$ appears in  $G$.

A causal graph $G$ is {\emph{causally sufficient}
if no latent vertices affect two or more vertices contained in $G$}
\citep{pearl2000causality,eberhardt2007interventions}.
In this paper, we assume that the causal graphs under consideration satisfy the causal  sufficiency.
A causal network model contains a DAG ${G}
=({  X}, E)$ and a joint distribution  $P$ over $X$.
Let $x_k$ and $x_l$ be two distinct vertices in $X$ and $S$ be a subset
of ${X} \setminus \{x_k, x_l\}$.
We use $(x_k\ind x_l|S)_{P}$ to denote that $x_k$ and $x_l$ be conditionally
independent given $S$ according to the joint distribution $P$.
A causal network model $(G,P)$ satisfies \emph{causal Markov} if a variable $x$ is conditionally independent of its non-descendants given all of its parents; that is, $(x\ind nd(x)|pa(x))_P$ holds for the causal model $(G,P)$.

If a causal network model $(G,P)$ satisfies the causal Markov condition,
then the joint
distribution of $p$ variables ${X}$ can be factored as follows
\citep{pearl1995causal,spirtes2001causation}
\begin{equation}\label{factor}
P({X})=\prod_{i=1}^p P(x_i|pa_i),
\end{equation}
where $P(x_i|pa_i)$ is the
conditional probability of $x_i$ given   its parent set  $pa_i$.

From Equation (\ref{factor}), some conditional independencies for the joint distribution $P$ can be read from the DAG $G$.  The concept of d-separation is used to describe the relation of vertices in a DAG $G$. For any pair of vertices $x_k$ and $x_l$ in $X$, and a subset  $S
\subseteq {X} \setminus \{x_k, x_l\}$,  the set $S$      d-separates   $x_k$ and $x_l$ in $G$ implies   that the set $S$ blocks all
connections of a certain type between $x_k$ and $x_l$ in $G$, denoted by  $(x_k\ind x_l|S)_{G}$. The  exact definition of d-separation   can  be found in \citet{pearl1988probabilistic}.  To learn the causal DAG $G$ from an observed data set of the joint distribution $P$, one  often assumes that the causal models under consideration satisfy the faithfulness  assumption defined as follows.

\begin{assumption}\label{faithfulness0} {\it The  faithfulness assumption.}
We say that $ P(x_1,\cdots,x_p) $ is faithful to the DAG  $ {G} $ if, for any pair of vertices $x_k$ and $x_l$ in $X$, and a subset  $S
\subseteq {X} \setminus \{x_k, x_l\}$ ,   the set $S$  d-separates   $x_k$ and $x_l$ in $G$  if $(x_k\ind x_l|S)_{P}$ holds,
\end{assumption}

With the causal Markov condition and the faithfulness assumption, the set of all (conditional) independencies read from a causal sufficient graph $G$ is the same as that from the joint distribution $P$.
\emph{A Markov equivalence class} is a set of DAGs that encode the same set of conditional independencies. \citet{verma1990equivalence} shows that two DAGs are Markov equivalent if and only if they have the same skeleton and the same v-structures. Therefore, one can recover
{the Markov equivalence class of the underlying DAG
from the corresponding joint distribution $P$,
which can be represented by
the skeleton and v-structures of $G$.
A \emph{constraint-based learning} algorithm
tries to find a DAG using the conditional
independency testing.}
The PC algorithm \citep{spirtes1995learning,spirtes2001causation} is the most well-known constraint-based algorithm. {In this paper, we propose a structural learning approach based on the PC algorithm, which learning a Markov equivalence class
from the multiple intervention data sets with unknown manipulated targets,
and we theoretically discuss the identifiability
and the correctness of the local structures learned by our approach.
}

\subsection{Causal learning with interventions}

{Suppose that in} an intervention experiment, some vertices in  ${X}$ may
be the \emph{targets} of the intervention which are manipulated externally.
Several types of interventions have been studied in the literature \citep{pearl1995causal}. A \emph{hard} intervention cuts off the edges between its targets and their parents; and a \emph{soft} intervention just changes the conditional probabilities of the targets given their parents.
Let $M$ denote the set of targets in an intervention experiment, and the set $O={X}\setminus M$ be the set of observational variables.
When an intervention affects more than one targets (i.e. $|M|>1$),
we assume that the intervention changes the condition probability of each target separately. That is, the intervention {may delete some of arrows pointing at
these targets from the original DAG {or may change the conditional probabilities of
these targets.}}
Therefore, the post-intervention joint distribution of ${X }$ for such an intervention can be factorized   as 
\begin{equation}\label{cdis}
P'(x_1,\ldots,x_n) = \prod_{x_i\in O}P(x_i|pa_i) \times
\prod_{x_i\in M}P'(x_i|pa'_i),
\end{equation}
where $P(x_i|pa_i)$ is the same as the conditional probabilities of $X_i$ in Equation (\ref{factor}) if $X_i$ is not a target of intervention, and $P'(x_i|pa'_i)$ be the post-intervention conditional probability of $x_i $ given its revised parent set $pa'_i$ in the intervention experiment. We have that $pa'_i=\emptyset$ when the intervention on $x_i$ is hard, and $pa'_i=pa_i$ if the intervention on $x_i$ is null.
Notice that $pa'_i\subseteq pa_i$ in Equation (\ref{cdis}), that is, the intervention on $x_i$ might be partially soft and partially hard.

In general, interventions with known targets are informative to identify the causal network in a Markov equivalence class  \citep{eberhardt2005number,he2008active}. However, when we do not know the targets of the interventions, additional uncertainty is introduced and the interventions might be useless to identify the causal networks.
  Below we give an example that the causal structure is not identified
if the manipulated target of intervention is unknown, while the structure is identified
if the target is known. 

\textbf{{Example 1.}} Consider the two causally sufficient graphs with two vertices: $x_1 \to x_2$ and $x_1\leftarrow x_2$, which are Markov equivalent.
If we know that $x_1$ is the target in a hard intervention, under the faithfulness assumption, $x_1\to x_2$ is identified by $x_1\nind x_2$ from the intervention data. However, if we do not know which one of $x_1$ and $x_2$ is the target, then $x_1\nind x_2$ (or $x_1 \ind x_2$) cannot be used to identify which of $x_1\to x_2$ and $x_1\leftarrow x_2$ is true.

  In the next section, we will propose two methods for learning causal structures from multiple intervention data sets with unknown manipulated  targets.

\section{Causal structure learning from multiple interventions with unknown targets }\label{chaplearning}

Suppose that there are $m$ interventions. For the $j$th intervention, let $G_j$ be the DAG, $M_j$  the set of targets, $O_j={X} \setminus
M_j$ the set of observational variables. Let $E$ and $E_j$ be the edge sets of $G$ and $G_j$ respectively.  By Equation (\ref{cdis}), the post-intervention joint probability for the $j$th intervention can be formulated as, for $j=1, \ldots, m$,
\begin{equation}\label{cdisinterv}
P_j({X}) = \prod_{x_i\in O_j}P(x_i|pa_i) \times
\prod_{x_i\in M_j}P_j(x_i|pa_i^j),
\end{equation}
where $pa_i$ is the parent set of $x_i$ in $G$, $pa_i^j$ is the parent set of $x_i$ in $G_j$, and $pa_i^j \subseteq pa_i$.


In Section \ref{graphmerging}, a graph merging method is proposed for the case that each intervention has a large sample such that we can efficiently learn a DAG from each intervention data set.
In Section \ref{mixedchisq}, for the case that each intervention has a small sample, we pool all intervention data sets together to learn the causal network.

\subsection{The graph merging method for causal structure learning}\label{graphmerging}

Let ${\cal D}_j$ denote the   data set   from the  $j$th intervention for $j\in \{1,\cdots,m\}$. When the sample size of ${\cal D}_j$ is large enough to learn a graph efficiently, we learn the graph $G_j$ from ${\cal D}_j$, and then we construct an overall graph merging the $m$ learned graphs, $G_j$ for $j\in \{1,\cdots,m\}$. We give the details of this method in Algorithm \ref{alggm}.

\begin{algorithm}[H]
\caption{Structural learning by combing the graphs learned from multiple experiments}
  \label{alggm}
\KwIn{${\cal D}=\{{\cal D}_j, j=1,\ldots,m\}$, data sets from $m$ intervention experiments}
\KwOut{${  G'=(X, E', V')}$, a skeleton graph with v-structures.}
\For {$j=1 \  to \  m$}
{Learn a skeleton graph with v-structures $G_j'=({X}, E_j', V_j')$ from ${\cal D}_j$ via the PC algorithm.}
Combine $\{G_j', j=1,\ldots,m\}$ to a graph $G'=({X},E',V')$ where $E' = \cup_{j=1}^m E_j$
and $V' = \cup_{j=1}^m V_j$. That is,   an edge (v-structure) appears in $E'$($V'$) if and only if  it is in    $E_j(V_j)$ for some $j \in \{1,\cdots,m\}$. \\
\Return{$G'$}
\end{algorithm}

Let $P_j$ be the underlying joint probability for the $j$th intervention, which can be formalized by Equation (\ref{cdisinterv}). Before showing the correctness of Algorithm \ref{alggm},
we describe the following faithfulness assumption.

\begin{assumption}\label{supp0}  The joint probability $ P_j(x_1,\cdots,x_p) $ is faithful to the DAG  $ {G}_j  $ for any $j=1,\ldots,m$
\end{assumption}

 Let $G=({X},E,V)$ denote the skeleton and v-structures of a 
 DAG
with an undirected edge set $E$ and a v-structure set $V$,  $G_j({X},E_j,V_j)$ denote the skeleton and v-structures of the post-intervention graph for the $j$th intervention, and  $P_j(x_1,\cdots,x_p)$ denote the post-intervention joint distribution for the $j$th intervention.  
Let {$G'=({X},E',V')$ and $G_j'=({X},E_j',V_j')$ denote the skeleton and v-structures of the corresponding graphs} learned by Algorithm \ref{alggm}.

\begin{theorem}\label{gmtheo}  {Let $P_j(x_1,\cdots,x_p)$ be defined by Equation (\ref{cdisinterv}), and  Assumption \ref{supp0} holds.
If there are no statistical errors for testing conditional independencies,}
then we have
\begin{enumerate}
  \item    $E_j'=E_j$ and $V_j'=V_j$  for $j=1, \ldots, m$,
   \item   $E'\subseteq E$, and
  \item  all directed edges in $G'$ appear in the underlying graph $G$.
      \end{enumerate}
\end{theorem}

\begin{proof}
Because $P_j(x_1,\cdots,x_p)$ follows Equation (\ref{cdisinterv}) and   faithfulness defined in Assumption \ref{supp0} holds for $j=1,\cdots,m$.  The joint probabilities and the underlying causal graph of $j$th intervention,  $P_j$ and $G_j$,   satisfy the  conditions of Markov properties and the assumption of  faithfulness,  and further  encode the same conditional independencies. According to the general results of constraint-based causal learning, we can recover the   Markov equivalence class of  $G_j$. That is, we can identify the skeleton and v-structures of $G_j$ correctly. It leads to  $E_j'=E_j$ and $V_j=V_j'$  for any $j\in \{1,\cdots,m\}$.

According to the Equation  (\ref{cdisinterv}), $E_j$ is a subset of $E$, that is $E_j\subseteq E$. We have  that $E'\subseteq E$ since $E'=\bigcup E_j'$ and $E_j'=E_j$.

Because $V_j'=V_j$ and all directed edges in $V_j$ also appear in the underlying graph $G$, we have that all  directed edges in $V'$ ($=\bigcup V_j'$), also appear in $G$.
\end{proof}
  Theorem \ref{gmtheo} shows that we  can learn  the skeleton and v-structures     $(E_j,V_j)$ for  $ G_j,j=1,\cdots,m\}$ correctly and  all edges and all directions of the learned graph $G'$ are true,
 but some edges and some v-structures in $G$ may be lost by the interventions.
 Clearly, if an edge is cut off in every intervention, we cannot recover it in  $G'$.  Similarly, a v-structure might be missed in $G'$ if it is removed by all interventions.

 {We say that a set of interventions is \emph{conservative}
if  for every vertex,  there is at least one intervention which does not affect the vertex (\citet{hauser2012characterization}).}

\begin{assumption}\label{conservative}
The set of $m$ interventions is conservative.
\end{assumption}

{\begin{proposition}\label{gmtheo}  Let $P_j(x_1,\cdots,x_p)$ be defined by Equation (\ref{cdisinterv}), and Assumptions \ref{supp0}  and \ref{conservative} hold.
Then the skeleton and v-structures of the underlying DAG $G$ can be recovered correctly by Algorithm \ref{alggm}
if there are no statistical testing errors for conditional independencies.
\end{proposition}
}

\begin{proof}
For any edge in $G$, say $x_l\to x_h$, suppose that $x_h$ is not a target in the $j-$th intervention, and then we have that the post intervention graph $G_j$ contains the edge $x_l\to x_h$. Algorithm \ref{alggm} can learn the edge between $x_l$ and $x_h$ from Theorem \ref{gmtheo}. Similarly, for any v-structure $x_l\to x_h\leftarrow x_r$,  suppose that $x_h$ is not a target in $j-$th intervention, and then we have that the post intervention graph $G_j$ contains the  v-structure $x_l\to x_h\leftarrow x_r$. According to Theorem \ref{gmtheo}, the set of learned v-structures $V'$  obtained by Algorithm \ref{alggm} contains  v-structure $x_l\to x_h\leftarrow x_r$. Therefore, we have that the skeleton and v-structures of the underlying graph $G$ can be recovered  via Algorithm \ref{alggm}.
\end{proof}

Several papers discuss the learning of graphical models from multiple data sets.  \citet{danks2002learning} and \citet{tillman2009integrating} propose some methods to learn a minimal equivalence class using the local independence information from distributed databases. \citet{lagani2012learning} propose to learn a causal network from multiple interventions by   multiple independence tests.  Comparing to our proposed  Algorithm  \ref{alggm}, their methods do not learn a graph for each experiment, so these methods might miss some specific causal structures in the individual interventions.

 \subsection{The data-pooling method for causal structure learning}\label{mixedchisq}

   When the sample size of each intervention is small,
   the conditional independence test by a single intervention data set becomes less powerful. In this section, we pool all intervention data sets together to learn a causal structure.

Let ${\cal D}_j$ denote the data set of the  $j$th intervention whose joint distribution is $P_j({X})$ in Equation (\ref{cdisinterv}). The data set $\cal D$ contains all $D_j$ for $j=1,\cdots,m$. Below we present a data-pooling learning method and its evaluation in Algorithm \ref{algdp}.


\begin{algorithm}[h]
\caption{Structural learning   by pooling all intervention data sets}
  \label{algdp}
\KwIn{${\cal D}=\{{\cal D}_j,j=1,\cdots,m\}$, data from all m interventions}
\KwOut{${G'}$, a graph with v-structures and edge set;  and the frequencies of edges learned by re-sampling.}
\Begin(Meta Learning){Learn a skeleton graph ${G'}=(X, E', V')$ with v-structures
using the conditional independence tests
in which the pooled data set $\cal D$ is used via the PC algorithm.\\}
\Begin(Evaluation ){
  \For {$i=1 \ to \ K$} {Draw a subset ${\mathbb I}_i$
randomly from $\{1,\cdots,m\}$ without replacement.\\
Pooling the   data together from the drawn   data sets:
$D_{{\mathbb I}_i} = \cup_{j \in {\mathbb I}_i} {\cal D}_j$. \\
Learn a skeleton graph ${G}_i'=(X, E_i' )$
from  the pooled data set $D_{{\mathbb I}_i}$.}
Let $E^*=\bigcup_{i=1}^K E_i'$ and  $\bar{E'}=E^*\setminus E'$. \\
For any edge in $E'$ and $\bar{E'}$, calculate the frequency of the edge appearing in $\{{G}_i',{i=1,\cdots,K}\}$.\\ \label{eav}}

 \Return{${G}'$,  the frequencies of edges in    $E'$ and  $\bar{E'}$}
\end{algorithm}

In Algorithm \ref{algdp}, we first give a meta learning of the underlying graph from the pooled data in Step 1, and then evaluate it with an intervention sampling technique.  We will show the correction of the meta learning in Theorem \ref{pdtheo},  and then we discuss  the evaluation of edges according to their frequencies.

{Let $I$ be a categorical variable with $m$ values $\{1,\ldots, m\}$
to indicate m interventions,
and the probability distribution $P(I=j)=p_j$, $p_j> 0$ and $\sum_{j=1}^m p_j = 1$.
Suppose that the data set $\cal D$ is generated as follows:
{(1) generate the frequencies $(n_1, \ldots, n_m)$ of $I$ from the probabilities
and $N=\sum_j n_j$, and}
(2) draw a data set $D_j$ of $X$ with the sample size $n_j$ from
the joint distribution $P_j$ defined in Equation (\ref{cdisinterv}).  Clearly, the data in  ${\cal D}=\bigcup_{j=1}^m D_j$ are independently and identically distributed (iid) {from the mixture joint distribution}
\begin{equation}\label{mdisinterv}
P_{M}({X}) = \sum_{j=1}^mP({X}|I=j)P(I=j)=\sum_{j=1}^mp_jP_j({X}).
\end{equation}
}

 Define $M=\bigcup_{j=1}^m M_j$ and $O=\bigcap_{j=1}^m O_j = {X}\setminus M$, that is, $M$ is the set of all manipulated  targets, and $O$ is the common observational variables in all interventions. We give the Markov properties of $P_{M}$ with respect to the underlying graph $G $ as follows.

\begin{mlemma} \label{markov}
For any $x_i \in O$, the vertex $x_i$ is independent of the
vertex set $nd(x_i)$ given the parent set $pa_i$ under the mixture
distribution $P_{M}$, denoted as $(x_i \ind nd(x_i)|pa_i)_{P_{M}}$.
\end{mlemma}

\begin{proof}
 For any $x_i \in O$, we first
obtain an order as $(x_{i_1},\cdots,x_i)$ in which all non-descendants of
$x_i$  are ranked before  $x_i$ and every vertex is behind of its parents in
the sequence as follows 
\begin{enumerate}
\item  let $L=X$,$Y=()$;
\item  if there exists
non-descendant $x$ of $x_i$ and there is no parents of $x$ in $L$, add $x$ to the end of $Y$ and $L=L\backslash \{x\}$; 
\item repeat 2, until $L$ only
contains $x_i$ and its descendants;
\item add $x_i$ to the end of $Y$.
\end{enumerate}

 Denoting $Y=(x_{i_1},\cdots,x_i)$,  we have $nd(x_i)\subseteq Y$ and

$$\begin{array}{rl}P(Y,I=j)=&P(Y|I=j)P(I=j)=P(Y|I=j)p_j\\
=&\prod_{x_k\in O_j\cap Y}P(x_k|pa(x_k))\prod_{x_k\in M_j\cap
Y}P_j(x_k|pa(x_k))p_j
\\
=&P(x_i|pa(x_i))\prod_{x_k\in \{O_j\cap Y\}\backslash
x_i}P(x_k|pa(x_k))\prod_{x_k\in M_j\cap Y}P_j(x_k|pa(x_k))p_j
\\
=&P(x_i|pa(x_i))f(\{Y\backslash x_i,I=j\} ),
\\
\end{array}$$
where $f(\{Y\backslash x_i,I\} )=\prod_{x_k\in \{O_j\cap
Y\}\backslash x_i}P(x_k|pa(x_k))\prod_{x_k\in M_j\cap
Y}P_j(x_k|pa(x_k))p_j$.

Let $Z=\{x_i,nd(x_i) \backslash pa(x_i),pa(x_i)\}$. From the construction of $Y$, we have $Z\subseteq Y$ and
$$\begin{array}{rl}P_{M}(Z)=
&\sum_{j=1}^m\sum_{x_k\notin
\{x_i,nd(x_i),pa(x_i)\}}P(Y,I=j)\\
=&\sum_{j=1}^m\sum_{x_k\notin
\{x_i,nd(x_i),pa(x_i)\}}P(x_i|pa(x_i))f(\{Y\backslash x_i,I=j\})
\\
=&P(x_i|pa(x_i))g(nd(x_i)\backslash pa(x_i),pa(x_i)),
\\
\end{array}$$
where $g(nd(x_i)\backslash pa(x_i),pa(x_i))=\sum_{j=1}^m\sum_{x_k\notin
\{x_i,nd(x_i),pa(x_i)\}}f(\{Y\backslash x_i,I=j\}).$
Therefore,
$$
P_{M}(nd(x_i),pa(x_i))=\sum_{x_i}P_{M}(Z)=g(nd(x_i)\backslash pa(x_i),pa(x_i)).
$$
We have
$$P_{M}(Z)=P(x_i|pa(x_i))P_{M}(nd(x_i)\backslash pa(x_i),pa(x_i))=P(x_i|pa(x_i))P_{M}(nd(x_i)|pa(x_i))P_{M}(pa(x_i)).$$
Thus we get,
$$P_{M}(nd(x_i),x_i|pa(x_i))=\frac{P_{M}(Z)}{P_{M}(pa(x_i))}=P(x_i|pa(x_i))P_{M}(nd(x_i)|pa(x_i)).$$
That is, we have $(v\ind
 nd(x_i)|pa(x_i))_{P_{M}}.$
 \end{proof}

 \citet{eberhardt2008sufficient}  shows that two data sets from two interventions with different targets can be pooled to test $x_i\ind x_j|S$ if $S$ can separate $(x_i,x_j)$ from the targets in each intervention.  Below we give an example to show the difference between  Eberhardt's condition and Lemma \ref{markov}.

\textbf{Example 2.} Consider a DAG with edges $x_1\to x_2\to x_3\to x_4\to x_5$ and $x_1\to x_5$. We implement two interventions, one on $x_1$ and the other on $x_5$.  $x_4\ind x_2|x_3$ can be confirmed   by the pooling data set $\cal D$ of two interventions.  From Lemma \ref{markov}, we have $x_4\ind x_2|x_3$ by $P_{M}$. However, this does not satisfy the condition required by Eberhardt since $x_3$ does not separate $(x_2,x_4)$ from $x_1$.

 We introduce another Markov property of  the mixture joint distribution with respect to a DAG $G$.

\begin{mlemma} \label{markov2}
For any non-adjacent pair $\{x_i, x_j\}$ which is contained in $O$,
there is a subset $S$ of ${X}$  such
that $(x_i \ind x_j | S)_{P_{M}}$.
\end{mlemma}

\begin{proof} For any pair $(x_i,x_j)$ of
non-adjacent vertices belong to $O$ in causal graph $G$, without of generality,
we suppose that $x_j$ is non-descendant of $x_i$, and then $x_j\in
nd(x_i)$. Letting $S=pa(x_i)$, we have $x_j\notin S$ and $(x_i\ind
x_j|S)_{P_{M}}$ according to Lemma \ref{markov}.
\end{proof}

This result means that the mixture distribution $P_{M}$ has the
pairwise Markov property for any pair of non-adjacent vertices in the
observational set $O$. According to this lemma, there is an
 edge connecting vertices $x_i$ and $x_j$ in $O$ if
$(x_i \nind x_i | S)_{P_{M}}$ for any subset $S$ of ${X}$.
Similar to the traditional constraint-based methods,
we also need  the concept of faithfulness assumption to show the correctness  of a causal learning method.

\begin{assumption}\label{supp1} {\it The faithfulness assumption of $P_{M}$ to $G$ over the observational set $O$.}
{We say under Assumption 3 of the conservative interventions}
that $P_{M}$ is faithful to the network ${G}$ over the
observation set $O$ if, for any pair of vertices $x_i$ and $x_j$ in
$O$, there is a set $S$ which d-separates the vertices $x_i$ and $x_j$
(denoted as $(x_i \ind x_j|S)_{G}$) when the conditional
independence $(x_i\ind x_j|S)_{P_{M}}$ holds, where $S$ is a subset
of ${X} \setminus \{x_i, x_j\}$.
\end{assumption}

Unlike Assumption 1, we do not require that the faithfulness assumption holds
for vertices in the target set $M$ since spurious independencies among
these vertices may be introduced due to the interventions. Under Assumption \ref{supp1}, we present the following two
results which ensure that edges and v-structures contained in the
observational set $O$ can be discovered correctly.

\begin{mlemma}\label{theorem2}
For a pair of vertices $x_i$ and $x_j$ contained in $O$,    under Assumption \ref{supp1}, $x_i$ and $x_j$ are
adjacent in a DAG ${G}$ if and only if
$(x_i\nind x_j|S)_{P_{M}}$ holds  for
any subset $S$ of ${X}$.
\end{mlemma}

\begin{proof}   For any $x_i\in O$, and $x_j\in O$ , if $x_i$ and $x_j$ are
not adjacent in $G$, from Lemma \ref{markov2}, we can get a
subset $S$ such that $(x_i\ind x_j|S)_{P_{M}}$. If there is no subset
$S\subseteq {X}$ such that $(x_i\ind x_j|S)_{P_{M}}$, then $x_i$ and
$x_j$ are adjacent in $G$. Because
 $P_{M}$ is faithful to $G$ over   $O$ according to Assumption \ref{supp1},   if there is a subset
 $S$ such that $(x_i\ind x_j|S)_{P_{M}}$ then $(x_i\ind x_j|S)_{G}$, that is $x_i$
 and $x_j$ is not adjacent in $G$. So if $x_i$ and $x_j$ are adjacent in $G$,
then there is no subset $S\subseteq {X}$ such that $(x_i\ind
x_j|S)_{P_{M}}$.
\end{proof}

\begin{mlemma}\label{theorem3}
Suppose that vertices $x_i$ and $x_k$ are adjacent and $x_j$, $x_k$ are
adjacent, but $x_i$ and $x_j$ are not adjacent in $G$ where
$x_i$, $x_j$ and $x_k$ are contained in $O$. Under Assumption \ref{supp1},   $x_i \rightarrow
x_k \leftarrow x_j$ is a subgraph of $G$ if and only if  $x_k
\in S$ implies $(x_i \nind x_j|S)_{P_{M}}$ for any subset $S$ of $X$.
\end{mlemma}

\begin{proof}
   For any $x_i\in O$ , $x_j\in O$ , if there is a subset $S$, where $x_k\in S$, such
that $(x_i\ind x_j|S)_{P_{M}}$, then $(x_i\ind x_j|S)_{G}$ according to Assumption \ref{supp1}. From the
definition of d-separation,  we can get that $x_i,x_k,x_j$ must not be
head to head, since  $x_i\rightarrow x_k \leftarrow x_j$ is subgraph of
$G$ and $x_k \in S$ implies $(x_i\nind x_j|S)_{P_{M}}$. On the other
hand, because $x_i$ and $x_k$, $x_k$ and $x_j$ are adjacent and $x_i$ and $x_j$
are not adjacent in $G$, if  $x_i\rightarrow x_k \leftarrow x_j$ is
not subgraph of $G$, then the structure of $x_i,x_k$ and $x_j$ should
be $x_i\rightarrow x_k \rightarrow j$, $x_i\leftarrow x_k \leftarrow x_j$ or
$x_i\leftarrow x_k \rightarrow x_j$, and thus  all d-separation set of $x_i$ and
$x_j$ must be including $x_k$. Since $x_i$ and $x_j$ are not adjacent,   there
are a set $S$  such that $(x_i\ind x_j|S)_{P_{M}}$, and then $(x_i \ind
x_j|S)_{G}$, thus we can get that $x_k \in S$. It means that if $x_k \in
S$ implies $(x_i\nind x_j|S)_{P_{M}}$, then $x_i\rightarrow x_k \leftarrow x_j$
is a subgraph of $G$.
\end{proof}

 Let $G_O $ be the induced subgraph of $G $ over the observational set $O$. With Lemmas \ref{markov}, \ref{theorem2} and \ref{theorem3},
we obtain the following main result of this section.

\begin{theorem}\label{pdtheo} Let $G'=(E',V')$ denote the graph obtained by the meta learning of Algorithm \ref{algdp}. If Assumption \ref{supp1} holds and there is no statistical errors in independence tests, then  all of edges and v-structure in $G_O'$ are exactly the same as those in $G_O$.
\end{theorem}

\begin{proof}  Under Assumption \ref{supp1}, we have that Lemma \ref{theorem2} and Lemma \ref{theorem3} hold. That is, $G_O$ can be learned correctly using the conditional independencies encoded in the underlying joint probabilities. Therefore, if there is no statistical errors in independency tests, the learned $G_O'$ is the same as $G_O$.
\end{proof}

 Theorem \ref{pdtheo} means that all edges and v-structures that are contained in the set $O$ can be discovered correctly from the mixture distribution $P_{M}$. However, between the vertices which are not contained in the set $O$, the mixture distribution $P_{M}$ may lead to spurious independencies and dependencies that do not encoded in the underlying DAG $G$. This is a cost of Algorithm \ref{algdp} for the cases in which the sample size of each intervention is small. Moreover, whether  a spurious dependence (or independence) appears in the mixture joint distribution $P_{M}$ depends on  how much it is  distorted by the interventions. In general, if only a small number of  interventions ``contaminate" the underlying independence (or dependence) of two vertices, it will keep in $P_{M}$. Therefore, many causal structures out of $O$ can be learned correctly.

Below,  we give two remarks about how to use the frequencies of edges obtained by  the re-sampling method to evaluate and improve the network learned by the meta learning of Algorithm \ref{algdp}.

\textbf{Remark.} The edges in $\bar{E'}$ are those which are not discovered from the original sample data set of all interventions
but are discovered from some re-sampling data sets.
If the frequency of such an edge $e$ is large,
it means that the edge $e$ may be missed by the spurious independency due to manipulating the relevant vertices in the interventions.
Thus we add the edge $e$ to $E'$ if its frequency appearing in $E'_i$'s is larger than a threshold.
We shall evaluate the re-sampling learning approach
using simulations in Section 4.

\section{Experimental study}\label{simulation}

In this section, we conduct three  simulations to evaluate the   proposed causal structure learning methods using the $37$-vertex Alarm network \citep{Beinlich}. We denote its causal graph as $G=(X,E)$, where $X=\{X_1,\cdots,X_{37}\}$.
{We illustrate and compare Algorithms \ref{alggm} and \ref{algdp}
in the first experiment,
then we study the performance of the different number of manipulated  targets
in the second one, and finally we discuss the re-sampling in Algorithm \ref{algdp} in the third one.}

 In each simulation, we generate artificial data as follows. We first generate $m$ interventions. In each intervention, we randomly choose   some vertices  as   targets to be manipulated, and the probability that an edge between a target and its parents is cut off is set to $0.5$.  The post-intervention conditional probabilities of each target are then generated from an uninformative Dirichlet distribution.  We finally generate a sample of size $n$ in each intervention, so there are $mn$ individuals in all interventions.
  The conditional probabilities for the underlying ALARM network
are from  \citet{Beinlich}

{We learn the skeleton and v-structures
from the artificial data generated from an underlying graph
using Algorithm \ref{alggm} or Algorithm \ref{algdp},
in which the PC algorithm is {used} and the conditional independencies
are checked by $\chi^2$ testing at a significance level   $\alpha =1\%$.  
Let $TP$ be the number of true positive edges,
$FN$ the number of false negative edges,
and $FP$ the false positive edges regardless of the edge directions;
and  let $TP_1$ be the number of true positive   directed edges or arrows,
$FN_1$ the number of false negative directed edges, and $FP_1$ the false positive directed edges, where directed edges are limited only to those in v-structures.
The true positive rate (TPR) and the true discovery rate (TDR) are defined as $TP/(TP+FP)$ and $TP/(TP+FN)$, respectively;
and the true positive rate D-TPR and  the true discovery rate D-TDR of directed edges in v-structures are defined as $TP_1/(TP_1+FP_1)$ and $TP_1/(TP_1+FN_1)$, respectively.}

In the first experiment, {for each intervention, we first generate randomly
an integer $k$ from $\{1,2,3,4,5\}$ with the same probabilities,
and we choose randomly $k$ vertices from the vertex set of ALARM as the
targets to be manipulated in the intervention.
Next we generate a sample of size $m \times n=5000$ for the multiple interventions
for each of four cases of $(n=2500, m=2)$, $(n=500,m=10)$, $(n=200,m=25)$  and $(n=100,m=50)$.
Then, we learn the skeleton and v-structures from each generated data set
using Algorithm \ref{alggm} and the meta learning of  Algorithm \ref{algdp}.
We repeat 100 simulations for each case,
and give TPR, TDR, D-TPR and D-TDR in Figure \ref{learnres}.
We can see that Algorithm \ref{alggm} works well
when the sample sizes are large for all interventions,
and that it works worse for small samples.
Algorithm \ref{algdp} works better than Algorithm \ref{alggm}
when the sample sizes are relative small.
}

\begin{figure}[h]
\centering
\includegraphics[scale=0.5]{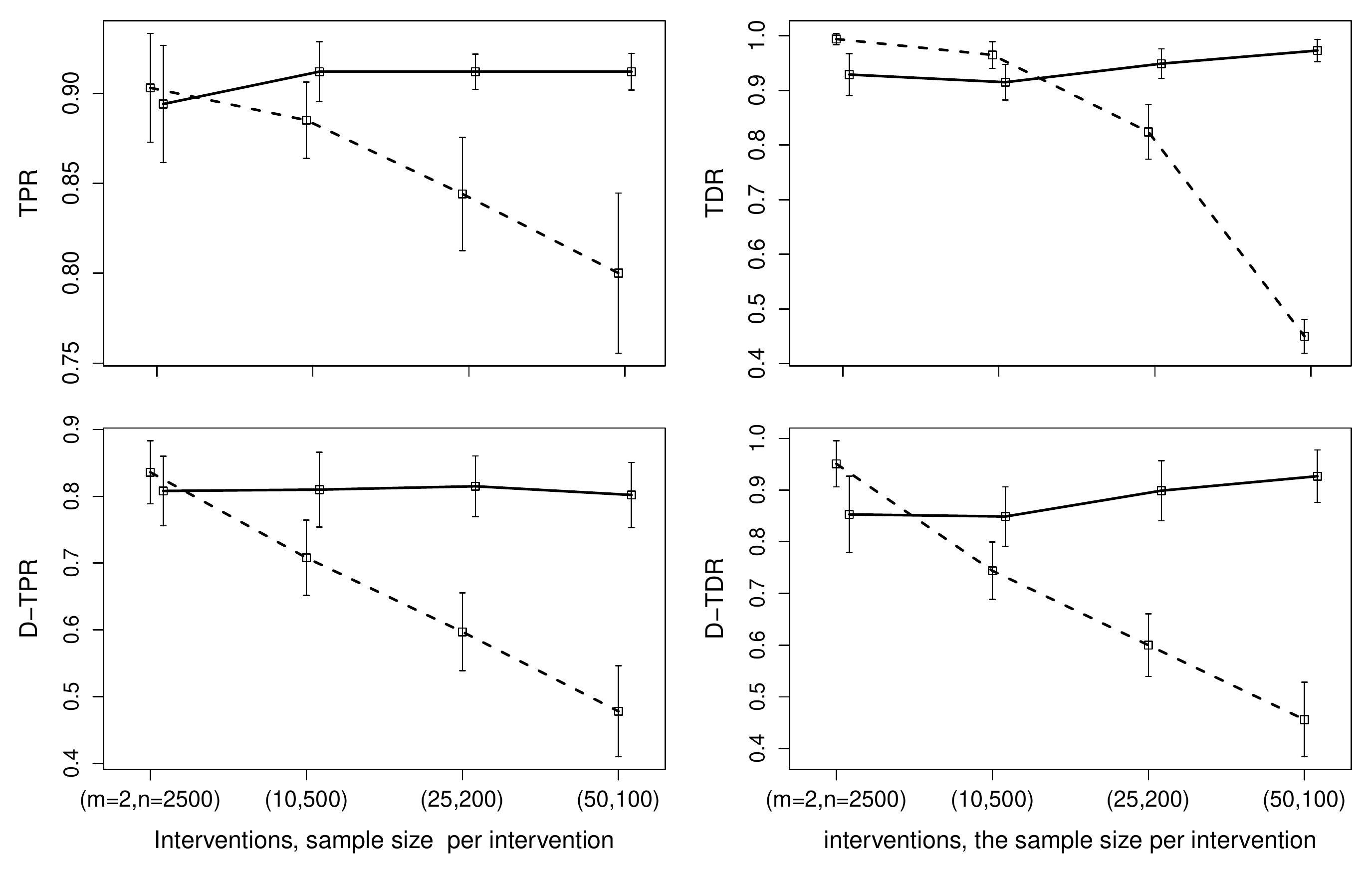}
\caption{TPR, TDR, D-TPR and D-TDR  of structural learning with Algorithm \ref{alggm} (dash lines) and  with Algirithm \ref{algdp} (solid lines). The top, median, and bottom of each error bar display  mean plus standard error,  mean, and  mean minus standard error, respectively.} \label{learnres}
\end{figure}

{From Figure \ref{learnres}, we can also see that with the increasing of the number $m$ of intervention data sets from $2$ to $50$, Algorithm \ref{algdp} has  larger true discovery rate TDR. It coincides with our discussion about the interventions with unknown target below Theorem \ref{pdtheo} in Section \ref{mixedchisq}. Since targets are chosen randomly, the distribution of the manipulated targets are more uniform over all vertices for the case of $(n=100, m=50)$ than those for other cases,
and thus each target is manipulated fewer times.
So the distortion of the dependencies relevant to a manipulated vertex
is much weaker than that for the case of $(n=2500, m=2)$,
in which the chosen target is manipulated  in at least a half of samples (2500).  

In the second experiment, we generate the data sets only for $(n=100, m=50)$.
We set the numbers of manipulated targets in all interventions to be a constant $C$
for a case, and we use four constants $C=2, 5, 10, 20$ as four cases.
We apply the meta learning of Algorithm \ref{algdp} to the generated data sets.
We repeat 100 simulations for each case
and report the means of TPR, TDR, D-TPR and D-TDR in Figure \ref{targetexp}.
We can see that the performance of Algorithm \ref{algdp} becomes worse for
the case of a larger constant $C$.
 It means that the larger the number of manipulated variables is,
the worse the learning performance is. 
}

\begin{figure}[h]
\centering
\includegraphics[scale=0.4]{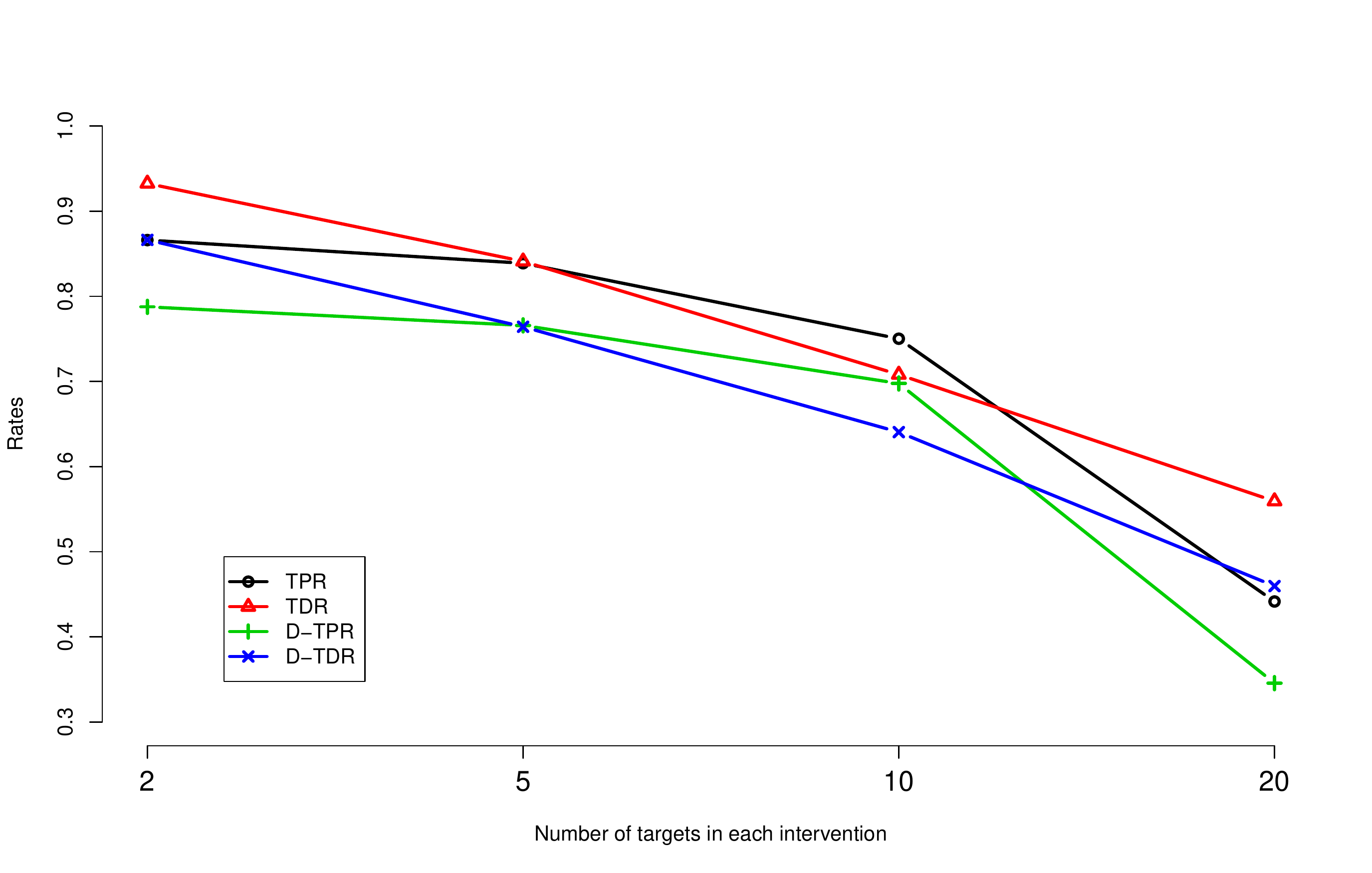}
\caption{TPR, TDR, D-TPR and D-TDR  of causal learning with Algorithm \ref{algdp} with different numbers of manipulated targets.} \label{targetexp}
\end{figure}

{
In the third experiment, we apply the re-sampling technique to the interventions
but not to the individuals in the original sample.
Since an intervention deletes some of edges,
it may make some of spurious independencies.
Making use of the re-sampling of the interventions,
we can see which edges are frequently found from the re-sampling data sets
and which are infrequently found.
Especially we should add these edges to the graph
which are frequently found
but are not found from the original data set of all interventions.
 In this experiment, we use the case of $(n=100,m=50)$ to illustrate
the performance of the re-sampling learning approach.
Let $G'=(X,E')$ denote the graph learned from the original data set
of all interventions by Algorithm \ref{algdp}.
We repeat $K=100$ re-samplings,
and we randomly draws 30 interventions from $m=50$ interventions
as the $i$th re-sampling data set for $i=1, \ldots, K$.
From the $i$th re-sampling data set, we learn the $i$th graph $G_i'=(X,E_i')$.
Let $E^*=\bigcup_{i=1}^{100} E_i'$, the union of the edge sets $E_i'$;
and $\bar{E'}=E^* \setminus E'$,
the set of edges that appear in $E^*$ but not in $E'$.
{For each edge $e$ in $E'\bigcup E^*$, we count the frequency $f(e)$ of each edge $e$ appearing in
all $E'_i$'s,
that is, $f(e)=\sum_{i=1}^{K} I_{e\in E_i'}$,
where $I(\cdot)$ is an indicator  function.}
Intuitively, for an edge $e \not\in E'$, we should add $e$ to $E'$ if $f(e)$ is larger; but for an edge $e \in E'$, we may or may not remove $e$
from $E'$ even if $f(e)$ is smaller.}
{In Table \ref{boots}, we show the frequencies of some edges in $E^*$.
In the upper part of Table 1,
for these edges in $\bar{E'}$,
we give 10 edges which have   the   largest frequencies among all edges in $E^*$,
these edges labelled `True' are the true edges,
and these labelled `False' are the false edges.
We can see that the top 3 edges with frequencies $\geq 68$ should be added to $E'$,
and other edges have frequencies $\leq 35$.
In the lower part of Table 1,
for these edges in $E'$, we show 10 edges which have the smallest frequencies.
We can see that the top 4 edges with frequencies $\leq 50$ have 2 false edges.

We repeat the above process for 100 times
in which we first generate a sample of the underlying ALARM network for the case $(n=100,m=50)$, next apply the re-sampling technique to the sample,
and then obtain a frequency table like Table 1.
From the 100 frequency tables obtained from the 100 repetitions,
we calculate the frequencies of `True' edges in the corresponding orders,
and we give the results in Table \ref{eval}.
From the upper part of Table \ref{eval},
we can see
that an edge $e \in \bar{E}'$ with a larger $f(e)$
has a larger frequencies to be a true edge.
From the lower part of Table \ref{eval},
we can see
that an edge $e \in E'$ with a smaller $f(e)$
has a smaller frequencies to be a true edge.

%

Our goal of the re-sampling is to find  the spurious independencies
due to the interventions,
we consider only to `add edges to $E'$' but not to `remove edges from $E'$.
Now we consider the threshold $\theta$ of
the frequencies $f(e)$'s for adding edges to $E'$.
After obtaining $G'=(X, E')$ learned by Algorithm \ref{algdp},
we add an edge $e$ in $\bar{E'}$ to $E'$ if $f(e) > \theta$.
Below we show the simulation results for various thresholds.
Let FN denote the average number of false negative edges
and FP the average number of false positive edges in the 100 learned graphs.
We give FNs and FPs for different $\theta$ in Table \ref{prunegrow}.
For the simulation results,
the best threshold is $\theta=20$
since the sum of FN=3.37 and FP=1.32 is the least.
A suitable threshold may be between 20 and 50,
for which not so many edges are added to the learned graph.


}

\begin{table}[h!]
\caption{The 10 edges having the largest frequencies in the $f(e)$ descending ($\downarrow$) order, and the 10 edges having the smallest frequencies in the $f(e)$ ascending ($\uparrow$) order. $i$-$j$ denotes a skeleton edge $e$ between vertices $i$ and $j$.} \label{boots}
\begin{center}
\begin{tabular}{ccccccccccc}  \hline
$f(e)\downarrow$   & 1 & 2 & 3 & 4& 5&6&7&8&9&10 \\ \hline
  $e \in \overline{E'}$  & 25-30 & 17-23 & 2-4 & 30-31 & 21-32 & 15-16 & 34-37 & 9-34 & 1-27 & 11-16 \\
   $f(e)$   &         79 &   77 &   68&    35&    10 &    2&     2&     2&     1&    1 \\
  $e\in E$  &True&True&True&False&False&True&False&False&False&True \\
 \hline \hline
$f(e)\uparrow$   & 1 & 2 & 3 & 4& 5&6&7&8&9&10\\\hline
    $e \in E'$   & 4-27  & 9-17  & 24-25 & 17-34 & 16-17  & 2-3 & 18-19 & 21-22 & 25-31 & 14-15   \\
    $f(e)$   &23    &24    &31   & 44   & 61    &89   & 90   & 90   & 90   & 96      \\
   $e\in E$  &  False &  True &  True &  False &    True &True&True&True&True&True  \\ \hline
\end{tabular}
\end{center}
\end{table}


\begin{table}[h!]
\caption{The frequencies of an edge at a top order
to be a true edge in 100 repetitions}\label{eval}
\begin{center}
\begin{tabular}{ccccccccccc}  \hline
$f(e)(\downarrow)$ order for $e \in \overline{E'}$ & 1 & 2 & 3 & 4& 5&6&7&8&9&10\\
Freq. to be a true edge in $E$        &95&76&64&41&30&24&13&8&7&7   \\  \hline
$f(e)(\uparrow)$ order for $e \in E'$ & 1 & 2 & 3 & 4& 5&6&7&8&9&10\\
Freq. to be a true edge in $E$     &55&77&91&96&99&100&100&100&100&100                    \\ \hline
\end{tabular}
\end{center}
\end{table}

\begin{table}
\caption{The simulation results of FN and PF
for different thresholds in the re-sampling learning}\label{prunegrow}
\begin{center}
\begin{tabular}{ crrrrrrrrrr} \hline
 $\theta$ &0& 5&   10 &15&   20  & 25&  30 &35&   50  & 100 \\ \hline
FN &  1.8 &2.64& 3.03&3.26& 3.37 &3.53& 3.73 &3.9&  4.09 &  5.13 \\
FP&9.82 &3.25& 2.08&1.58& 1.32 & 1.18&1.11&1.05&0.97 &  0.96 \\
Sum& 11.62 & 5.89  &5.11 & 4.84 & 4.69 &  4.71&  4.84 & 4.95&  5.06&  6.09\\  \hline
\end{tabular}
\end{center}
\end{table}

{
Finally we give the simulation results
of learning the ALARM network
from an observational data set without interventions.
In each simulation,
we first generate an observed data set of the same sample size 5000
from the distribution without interventions,
and then we apply the PC algorithm to learn a graph.
We repeat 100 simulations and obtain FN=5.51 and FP=0.22.
According to the sum of FN and FP,
this result (5.51+0.22) of learning from an observational data set
is better than that (5.13+0.96 for $\theta=100$) of
learning from a multiple intervention data set and worse than that (3.37+1.32 for $\theta=20$) of learning from the re-sampling   learning approach.   

}

\section{Discussion}\label{disccussion}

In this paper, we study how to learn causal structures from a data set with multiple interventions of unknown targets. Two approaches are presented to learn causal  structures from  large or small samples in each intervention. We show that the graph merging method works well when each intervention has a large enough sample to learn a graph efficiently, while the pooling data method is preferable when the   sample size in each intervention is small.

Algorithm \ref{alggm} assumes that there are no statistical  errors for testing conditional independencies. However, in a real scenario, two v-structures in two graphs learned by Algorithm \ref{alggm} may contain the same edge oriented in different ways because of statistical errors. This conflicting  problem due to statistical errors may also appear in the constrain-based algorithms even without interventions.
As treated in most algorithms, we can simply remove some v-structures inducing the conflicting constrains.
\citet{sofia2014constraint} propose an approach in which a function of their corresponding p-values is used to sort the constrains in order of confidence.
 
In Algorithm \ref{algdp}, we output the frequencies of  edges appearing in the graphs learned by re-sampling method to evaluate the skeleton of meta learning graph. Additionally, we also can evaluate  directed edges in v-structures learned in  Algorithm \ref{algdp} in a similar way.

Moreover, we use  R to implement algorithms and experiments in this paper and the R package  can be found at   \url{http://www.math.pku.edu.cn/teachers/heyb/}

\section*{Acknowledgements}
This research was supported by NSFC (11101008, 71271211, 11331011),
863 Program of China (2015AA020507) and
973 Program of China (2015CB856000).
We would like to thank the editor and two referees for their very valuable comments and suggestions.

\bibliography{ref1}

\begin{thebibliography}{26}
\providecommand{\natexlab}[1]{#1}
\providecommand{\url}[1]{\texttt{#1}}
\expandafter\ifx\csname urlstyle\endcsname\relax
  \providecommand{\doi}[1]{doi: #1}\else
  \providecommand{\doi}{doi: \begingroup \urlstyle{rm}\Url}\fi

\bibitem[Beinlich et~al.(1989)Beinlich, Suermondt, Chavez, and
  Cooper]{Beinlich}
I.A. Beinlich, M.~Suermondt, R.M. Chavez, and G.F. Cooper.
\newblock {A Case Study with two Probabilistic Inference Techniques for Belief
  Networks}.
\newblock In \emph{Aime 89: Proceedings: Second European Conference on
  Artificial Intelligence in Medicine, London, August 29th-31st 1989}, page
  247. Springer, 1989.

\bibitem[Cooper and Yoo(1999)]{cooper1999causal}
G.~Cooper and C.~Yoo.
\newblock {Causal discovery from a mixture of experimental and observational
  data}.
\newblock In \emph{Proc. Fifthteenth Conference on Uncertainty in Artificial
  Intelligence (UAI¡¯99)}, pages 116--125. Citeseer, 1999.

\bibitem[Danks(2002)]{danks2002learning}
D.~Danks.
\newblock {Learning the Causal Structure of Overlapping Variable Sets}.
\newblock In \emph{Proceedings of the 5th International Conference on Discovery
  Science}, pages 178--191. Springer-Verlag London, UK, 2002.

\bibitem[Eaton and Murphy(2007)]{eaton2007exact}
D.~Eaton and K.~Murphy.
\newblock {Exact Bayesian structure learning from uncertain interventions}.
\newblock In \emph{AI and Statistics}, pages 107--114. Citeseer, 2007.

\bibitem[Eberhardt(2006)]{eberhardt2006sufficient}
F.~Eberhardt.
\newblock {Sufficient condition for pooling data from different distributions}.
\newblock In \emph{First Symposium on Philosophy, History, and Methodology of
  Error}. Citeseer, 2006.

\bibitem[Eberhardt(2008)]{eberhardt2008sufficient}
F.~Eberhardt.
\newblock A sufficient condition for pooling data.
\newblock \emph{Synthese}, 163\penalty0 (3):\penalty0 433--442, 2008.

\bibitem[Eberhardt and Scheines(2007)]{eberhardt2007interventions}
F.~Eberhardt and R.~Scheines.
\newblock {Interventions and causal inference}.
\newblock \emph{Philosophy of Science}, 74\penalty0 (5):\penalty0 981--995,
  2007.

\bibitem[Eberhardt et~al.(2005)Eberhardt, Glymour, and
  Scheines]{eberhardt2005number}
F.~Eberhardt, C.~Glymour, and R.~Scheines.
\newblock {On the number of experiments sufficient and in the worst case
  necessary to identify all causal relations among n variables}.
\newblock In \emph{Proc. of the 21st Conference on Uncertainty in Artificial
  Intelligence (UAI)}, pages 178--183. Citeseer, 2005.

\bibitem[Finegold and Drton(2011)]{finegold2011robust}
M.~Finegold and M.~Drton.
\newblock Robust graphical modeling of gene networks using classical and
  alternative t-distributions.
\newblock \emph{The Annals of Applied Statistics}, 5\penalty0 (2A):\penalty0
  1057--1080, 2011.

\bibitem[Friedman(2004)]{Friedman}
N.~Friedman.
\newblock {Inferring cellular networks using probabilistic graphical models}.
\newblock \emph{Science Signaling}, 303\penalty0 (5659):\penalty0 799, 2004.

\bibitem[Geng et~al.(2004)Geng, He, and Wang]{geng2004relationship}
Zhi Geng, Yangbo He, and Xueli Wang.
\newblock Relationship of causal effects in a causal chain and related
  inference.
\newblock \emph{Science in China Series A: Mathematics}, 47\penalty0
  (5):\penalty0 730--740, 2004.

\bibitem[Hauser and B{\"u}hlmann(2012)]{hauser2012characterization}
A.~Hauser and P.~B{\"u}hlmann.
\newblock Characterization and greedy learning of interventional markov
  equivalence classes of directed acyclic graphs.
\newblock \emph{The Journal of Machine Learning Research}, 13\penalty0
  (1):\penalty0 2409--2464, 2012.

\bibitem[He and Geng(2008)]{he2008active}
Yangbo He and Zhi Geng.
\newblock Active learning of causal networks with intervention experiments and
  optimal designs.
\newblock \emph{Journal of Machine Learning Research}, 9:\penalty0 2523--2547,
  2008.

\bibitem[Heckerman et~al.(1999)Heckerman, Meek, and
  Cooper]{heckerman1999bayesian}
D.~Heckerman, C.~Meek, and G.~Cooper.
\newblock {A Bayesian approach to causal discovery}.
\newblock \emph{Computation, causation, and discovery}, pages 143--67, 1999.

\bibitem[Jansen et~al.(2003)Jansen, Yu, Greenbaum, Kluger, Krogan, Chung,
  Emili, Snyder, Greenblatt, and Gerstein]{jansen2003bayesian}
R.~Jansen, H.~Yu, D.~Greenbaum, Y.~Kluger, N.J. Krogan, S.~Chung, A.~Emili,
  M.~Snyder, J.F. Greenblatt, and M.~Gerstein.
\newblock {A Bayesian networks approach for predicting protein-protein
  interactions from genomic data}.
\newblock \emph{Science}, 302\penalty0 (5644):\penalty0 449, 2003.

\bibitem[Maathuis et~al.(2009)Maathuis, Kalisch, and
  B{\"u}hlmann]{maathuis2009estimating}
M.~H. Maathuis, M.~Kalisch, and P.~B{\"u}hlmann.
\newblock {Estimating high-dimensional intervention effects from observational
  data}.
\newblock \emph{The Annals of Statistics}, 37\penalty0 (6A):\penalty0
  3133--3164, 2009.
\newblock ISSN 0090-5364.

\bibitem[Mikko and Kismat(2004)]{koivisto2004exact}
K.~Mikko and S.~Kismat.
\newblock Exact bayesian structure discovery in bayesian networks.
\newblock \emph{The Journal of Machine Learning Research}, 5:\penalty0
  549--573, 2004.

\bibitem[Pearl(1995)]{pearl1995causal}
J.~Pearl.
\newblock {Causal diagrams for empirical research}.
\newblock \emph{Biometrika}, 82\penalty0 (4):\penalty0 669, 1995.

\bibitem[Pearl(2000)]{pearl2000causality}
J.~Pearl.
\newblock \emph{{Causality: Models, reasoning, and inference}}.
\newblock Cambridge Univ Pr, 2000.

\bibitem[Pearl and Shafer(1988)]{pearl1988probabilistic}
J.~Pearl and G.~Shafer.
\newblock \emph{{Probabilistic reasoning in intelligent systems: networks of
  plausible inference}}.
\newblock Morgan Kaufmann San Mateo, CA, 1988.

\bibitem[Spirtes and Meek(1995)]{spirtes1995learning}
P.~Spirtes and C.~Meek.
\newblock {Learning Bayesian networks with discrete variables from data}.
\newblock In \emph{Proceedings of the first international conference on
  knowledge discovery and data mining}, pages 294--299. Menlo Park, CA: AAAI,
  1995.

\bibitem[Spirtes et~al.(2001)Spirtes, Glymour, and
  Scheines]{spirtes2001causation}
P.~Spirtes, C.N. Glymour, and R.~Scheines.
\newblock \emph{{Causation, prediction, and search}}.
\newblock The MIT Press, 2001.

\bibitem[Tillman et~al.(2009)Tillman, Danks, and
  Glymour]{tillman2009integrating}
R.~E. Tillman, D.~Danks, and C.~Glymour.
\newblock {Integrating locally learned causal structures with overlapping
  variables}.
\newblock \emph{Advances in Neural Information Processing Systems}, 21, 2009.

\bibitem[Triantafillou and Tsamardinos(2014)]{sofia2014constraint}
S.~Triantafillou and I.~Tsamardinos.
\newblock Constraint-based causal discovery from multiple interventions over
  overlapping variable sets.
\newblock \emph{arXiv preprint arXiv:1403.2150}, 2014.

\bibitem[Verma and Pearl(1990)]{verma1990equivalence}
T.~Verma and J.~Pearl.
\newblock {Equivalence and synthesis of causal models}.
\newblock In \emph{Proceedings of the Sixth Annual Conference on Uncertainty in
  Artificial Intelligence}, page 270. Elsevier Science Inc., 1990.

\bibitem[Vincenzo et~al.(2012)Vincenzo, Ioannis, and Sofia]{lagani2012learning}
L.~Vincenzo, T.~Ioannis, and T.~Sofia.
\newblock Learning from mixture of experimental data: a constraint--based
  approach.
\newblock In \emph{Artificial Intelligence: Theories and Applications}, pages
  124--131. Springer, 2012.

\end{thebibliography}

 \end{document}